\documentclass{article}

\usepackage[final]{nips_2018}

\usepackage{graphicx}
\usepackage{amsmath}
\usepackage{mathtools}
\usepackage{amssymb}
\usepackage{amsthm}
\usepackage{amsfonts}
\usepackage{dsfont}
\usepackage{xcolor}
\usepackage{enumerate}
\usepackage{hyperref}
\usepackage{subcaption}
\usepackage{stmaryrd}
\usepackage{lineno}	
\usepackage[colorinlistoftodos]{todonotes} 
\usepackage{comment}
\usepackage{lmodern}
\usepackage{tabularx}

\newcolumntype{Y}{>{\centering\arraybackslash}X}

\usepackage[nobiblatex]{xurl}

\theoremstyle{plain}

\newtheorem{definition}{\textbf{Definition}}
\newtheorem{assumption}{\textbf{Assumption}}
\newtheorem{lemma}{\textbf{Lemma}}
\newtheorem{proposition}{\textbf{Proposition}}
\newtheorem{main_result}{Main result}
\newtheorem{remark}{Remark}
\newtheorem{remarque}[assumption]{\textbf{Remark}}
\newcommand{\sref}[2]{\hyperref[#2]{#1 \ref*{#2}}}

\newcommand{\qedwhite}{\hfill \ensuremath{\square}}
\newtheorem{application}{\textbf{Application}}

\begin{document}

\title{Universal hidden monotonic trend estimation with contrastive learning}

\author{
  Edouard Pineau \\
  \texttt{EthiFinance}\\
  \texttt{edouard.pineau@ethifinance.com}
  \And
  Sébastien Razakarivony \\
  \texttt{Safran}\\
  \texttt{sebastien.razakarivony@safrangroup.com}
}


\maketitle

\begin{abstract}
  In this paper, we describe a universal method for extracting the underlying monotonic trend factor from time series data. We propose an approach related to the Mann-Kendall test, a standard monotonic trend detection method and call it \textit{contrastive trend estimation} (CTE). We show that the CTE method identifies any hidden trend underlying temporal data while avoiding the standard assumptions used for monotonic trend identification. In particular, CTE can take any type of temporal data (vector, images, graphs, time series, etc.) as input. We finally illustrate the interest of our CTE method through several experiments on different types of data and problems. 
\end{abstract}



\section{Introduction}
\label{intro}


Our paper focuses on the estimation of a monotonic trend factor underlying temporal data. Such estimation is interesting in many fields, e.g., health monitoring \cite{pineau2020unsupervised}, survival analysis \cite{miller2011survival} or climate change monitoring \cite{Huang2022}. In all these fields and related trend estimation problems, we observe samples generated by a monitored system (e.g., an ageing mechanical system, a credit debtor, earth's weather and climate conditions) at different times in its life, and we assume that the state of the system drifts monotonically. These observed samples may be of any type (e.g., vectors, images, time series, graphs), depending on the monitored system. \autoref{fig:context} illustrates the general context of trend estimation.

\begin{figure}
    \centering
    \includegraphics[width=0.8\linewidth]{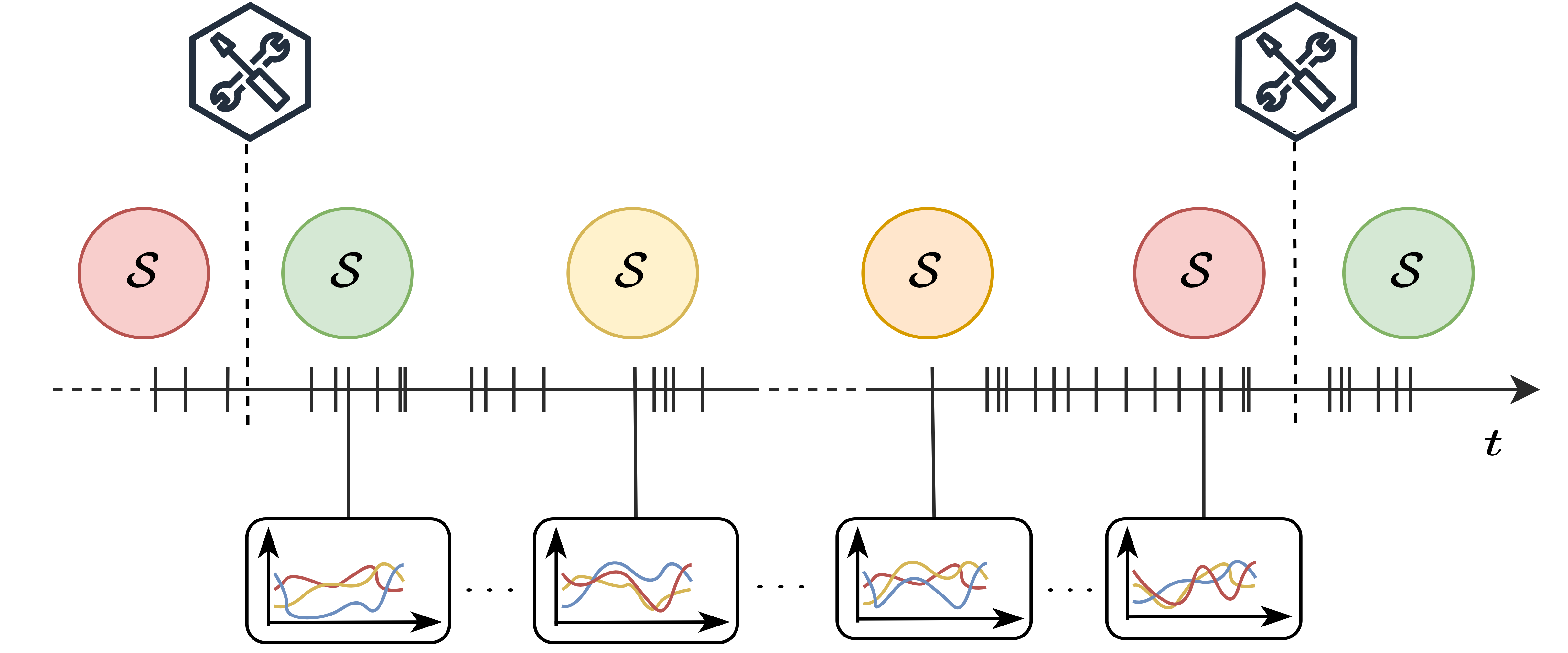}
    \caption{\textit{\color{black}
    Illustration of the context of the paper's contribution. We have a monitored system $\mathcal{S}$ that generates data samples (colored curves) at random time. The hidden trend $\tau$ underlying the system (colors from green to red) represents the hidden state of $\mathcal{S}$ that changes monotonically until a state restoration is applied (tools in hexagons): samples between two state restorations form a sequence with a monotonic hidden trend. The relation between trend and observed data may be an arbitrary function yet assumed to preserve the information about the trend.}}
    \label{fig:context}
\end{figure}

More generally, when studying temporal data, it is common to assume the existence of \textit{structural latent factors}, supposed meaningful, that generated the data \cite{harvey199310}. These components are generally allocated into four groups.
%
The \textit{trend} components are monotonic long-term signals. 
%
The \textit{cycle} components are factors exhibiting rises and falls that are not of a fixed frequency.
%
The \textit{seasonality} components are periodic patterns occurring at a fixed frequency.
%
The \textit{irregularity} factors represent the rest of the information (considered as a noise). 
We assume independent structural factors. The challenging yet essential task is the \textit{identification}
of one or several of these factors, that is called \textit{blind source separation} \cite{choi2005blind}, \textit{independent component analysis} \cite{hyvarinen2000independent} or \textit{disentanglement} \cite{bengio2013representation}. In this paper, the objective is to detect, isolate and identify only the trend component. \cite{hyvarinen2016unsupervised} shows that if we know one hidden component under time series data, we can find the others conditionally. Hence, finding the trend component is not only useful for many monitoring problems, it is relevant for further analysis.

Often, trend estimation methods seek monotonic variations in the values of the data or in expert-based statistics computed from data \cite{chen2017design,niknam2017techniques}. In practice, the trend can be deeply hidden in the data or may be not well defined because of a lack of information about the monitored system. Hence, we may not know which variable or statistics to follow to find the trend.

In this paper, we \textit{learn to infer} the trend factor from data (of any type) without labels or expert supervision, using only samples' time index. To do so, we develop a general method based on \textit{Contrastive Learning} (CL). CL recently received high interest in self-supervised representation learning \cite{le2020contrastive}, in particular for time series data (see, e.,g.,  \cite{franceschi2019unsupervised,banville2019self,wang2015unsupervised}). Our CL approach uses a loss inspired by Mann-Kendall test \cite{Mann1945}, a standard trend detection method. 

\vspace{0.3cm}

The rest of the paper presents our universal \textit{trend inference} method called \textit{Contrastive Trend Estimation} (CTE).
%
\sref{Section}{sec:cte_setup} presents the method. 
\sref{Section}{sec:ctd_identifiability} analyzes the theoretical foundation of our method in terms of identifiability. 
\sref{Section}{sec:related_work_ctd} lists related works on trend detection and estimation. 
\sref{Section}{sec:experiments} presents a set of experiments to illustrate the interest of our approach for trend estimation and survival analysis. 
Concluding remarks are presented in \sref{Section}{sec:conclusion}.

\section{Contrastive trend detection}
\label{sec:cte_setup}


\paragraph{\textbf{Notations.}} Let $X$ be a sequence of $N_X\in\mathds{N}$ observed samples generated by a monitored system denoted by $\mathcal{S}$. We assume that a hidden state of $\mathcal{S}$ drifts monotonically. We note $\mathcal{X}$ the dataset of all sequences $X$ in which there exist a hidden monotonic factor. We note $t_i$ the time index of the $i^{th}$ observed sample, $i \in \llbracket 1, N_X \rrbracket$. We assume that each sequence $X \in \mathcal{X}$ has been generated from structural factors through a function $F$, such that at least the information about the trend is not annihilated (in blind source separation problems, $F$ would be assumed invertible).
That is, for each $X$ there exist $Z^X \coloneqq (\tau^X, c^X, s^X, \epsilon^X)$ such that $X_{t_i} = F(Z^X_{t_i})$
%
%
,where  
$\tau^X$, $c^X$, $s^X$, and $\epsilon^X$ represent respectively (resp.) the monotonic trend, the cycle, the seasonality, and the irregularity that generated $X$. The paper's goal is to estimate the factor $\tau^X$ from $X$.

\paragraph{\textbf{Our CTE approach.}} For each $X \in \mathcal{X}$, we select two sampling times $t_u, t_v$ in $\{t_1, \dots, t_{N_X}\}^2$, such that, without loss of generality (w.l.o.g.), $t_u<t_v$. 
The value of the hidden trend at the sampling time $t$ for $X_t$ is noted $\tau^X_t$.
Since we do not have access to the \textit{true} hidden trend, we need assumptions about $\tau^X$. We use the natural Assumption \ref{hyp:monotonicity} to estimate $\tau^X$. 

\begin{assumption}
{\normalfont \textbf{(Monotonicity).}} For each sequence $X \in \mathcal{X}$ and all sample couples  $(X_{t_u},X_{t_v})$, we have that $t_u \leq t_v \Longleftrightarrow \tau^X_{t_u} \leq \tau^X_{t_v}$.
\label{hyp:monotonicity}
\qedwhite
\end{assumption}

\noindent To extract the trend component, we use a neural network (NN) $F_\phi$ with parameters $\phi$ that embeds each sample $X_t$ into a $d_e$-dimensional vector space, with which we define $g_{\beta, \phi}: \mathbb{R}^{d_e} \times \mathbb{R}^{d_e} \to [0,1]$ a parametric logistic regressor defined as follows:

\begin{equation}
    g_{\beta, \phi}\left(X_{t_u}, X_{t_v} \right)
    = \sigma\left( \beta^\top F_\phi(X_{t_v}) - \beta^\top F_\phi(X_{t_u}) \right)  \;,
    \label{eq:cte_model}
\end{equation}

\noindent where $\sigma(x) \coloneqq (1 + e^{-x})^{-1}$ is the sigmoid function. 
Let $C_{uv} \coloneqq \mathds{1}_{\{\tau^X_{t_u} \leq \tau^X_{t_v}\}}$ be the indicator function that describes the trend direction between $t_u$ and $t_v$ for any sample $X$. Under the \sref{Assumption}{hyp:monotonicity}, we have also $C_{uv}=\mathds{1}_{\{t_u \leq t_v\}}$. Then, we can build $C_{uv}$ from sample's time indices. We then can learn the posterior distribution $p(C_{uv} | X_{t_u}, X_{t_v})$, i.e., learn the identity:

\begin{equation}
\label{eq:perfect_g}
    p(C_{uv} = 1| X_{t_u}, X_{t_v}) = g_{\beta, \phi}\left(X_{t_u}, X_{t_v} \right) \;.
\end{equation}

\noindent As for common binary classification problems, training is done by minimizing the binary cross entropy (BCE) between $C_{uv}$ and the regressor $g_{\beta, \phi}\left(X_{t_u}, X_{t_v} \right)$, for all pairs of time indices $(t_u, t_v)$, $\forall X in \mathcal{X}$, i.e., by minimizing:

\begin{equation}
    R(\beta,\phi; \mathcal{X}) = -\,\mathbb{E}_{X \in \mathcal{X}} \left[ \sum_{i,j=1}^{N_X} {C_{ij} 
    \log \big(g_{\beta, \phi}
    \left(X_{t_i}, X_{t_j} \right)
    \big)} \right]\;.
\label{eq:ranking_problem}
\end{equation}

\begin{remarque}
Eq.~\eqref{eq:ranking_problem} is similar to the Mann-Kendall statistics of eq.~\eqref{eq:mk_statistics} presented in \sref{Section}{sec:related_work_ctd} of related work.
\qedwhite
\end{remarque}

Once the parameters $(\phi, \beta)$ are fitted, we build an estimator $\beta^\top F_\phi(X_{t})$ of the trend factor $\tau^X_t$. In the next section, we show in what extent this estimator effectively estimates the hidden trend factor.

\section{Identifiability study}
\label{sec:ctd_identifiability}

We assume that $F_\phi$ is a universal approximation function (e.g., a sufficiently large NN) and that the amount of data is large enough (equivalent to infinite data) such that we achieve the identity of eq.~\eqref{eq:perfect_g}. 

\begin{definition} 
{\normalfont \textbf{(Minimal sufficiency).}} A sufficient statistic $T$ is minimal sufficient if for any sufficient statistic $U$, there exists a function $h$ such that $T=h(U)$. If $U$ is also minimal, then $h$ is a bijection. 
\label{def:minimal_sufficiency}
\qedwhite
\end{definition}

\begin{proposition}
    $\beta^\top \left(F_\phi(X_{t_v}) - F_\phi(X_{t_u})\right)$ is a minimal sufficient statistic for trend label $C_{uv}$.
    \label{prop:minimal_sufficiency}
\end{proposition}

\begin{proof}
    
    First we remind that logistic regression learns likelihood ratios, i.e., $F_{\beta, \phi}$ is a log-likelihood difference. In fact, using the Bayes rule, we get
    
    \begin{equation}
        p(C_{uv} = 1 | X_{t_u}, X_{t_v}) = \frac{p(X_{t_u}, X_{t_v} | C_{uv} = 1)p(C_{uv} = 1)}{p(X_{t_u}, X_{t_v})}.
    \label{eq:bayes}
    \end{equation}
    
    \noindent Moreover, using properties of sigmoid function $\sigma$ and eq.~\eqref{eq:perfect_g}, we have
    
    \begin{equation}
        e^{\beta^\top \left(F_\phi(X_{t_v}) - F_\phi(X_{t_u})\right)}= \frac{p(C_{uv}=1 | X_{t_u}, X_{t_v})}{p(C_{uv}=0 | X_{t_u}, X_{t_v})}\;.
        \label{eq:sigmoid}
    \end{equation}

    \noindent Finally, from eq.\eqref{eq:bayes} and eq.~\eqref{eq:sigmoid} we obtain

    \begin{align*}
        e^{ \beta^\top \left(F_\phi(X_{t_v}) - F_\phi(X_{t_u}) \right)} = \frac{p(X_{t_u}, X_{t_v} | C_{uv}=1)p(C_{uv}=1)}{p(X_{t_u}, X_{t_v} | C_{uv}=0)p(C_{uv}=0)}.  
    \end{align*}

    \noindent We note that $p(C_{uv}=1) = p(C_{uv}=0)$, since we randomly choose $t_u$ and $t_v$ simultaneously within $\{1, \dots, N_X \}$. Hence,
    
    \begin{equation}
        e^{ \beta^\top \left(F_\phi(X_{t_v}) - F_\phi(X_{t_u}) \right)} = \frac{p(X_{t_u}, X_{t_v} | C_{uv}=1)}{p(X_{t_u}, X_{t_v} | C_{uv}=0)},
    \label{eq:likelihood_ratio}
    \end{equation}
    
    \noindent that is a likelihood ratio. \textit{Theorem~2} of \cite{goh2001econ} states that the density ratio is a minimal sufficient statistics. Then, from \sref{Definition}{def:minimal_sufficiency}, $\beta^\top \left(F_\phi(X_{t_v}) - F_\phi(X_{t_u})\right)$ is a minimal sufficient statistic for $C_{uv}$. 

\end{proof}

\begin{remarque}
    Learning the likelihood ratio of eq.~\eqref{eq:likelihood_ratio} without explicitly knowing the likelihood is called \textit{likelihood-free inference} \cite{thomas2016likelihood}. Related to our approach, \cite{gutmann2018likelihood} explains how classification can be used to do likelihood-free inference using CL. Compared to maximum likelihood approaches, its main advantage lies in the fact that contrasting two models enables to cancel out some computationally untractable terms (as the log-determinant of the Jacobian of the embedding function or the partition function of the model).
    \qedwhite
\end{remarque}

\begin{lemma}
    $\exists h$ monotonic such that $\beta^\top \left(F_\phi(X_{t_v}) - F_\phi(X_{t_u})\right) = h(\tau^X_{t_v} - \tau^X_{t_u})$. 
\label{lemma:minimal_sufficiency}
\end{lemma}

\begin{proof}
    $\tau^X_{t_v} - \tau^X_{t_u}$ is by definition a minimal sufficient statistic. Hence, since the estimator $\beta^\top \left(F_\phi(X_{t_v}) - F_\phi(X_{t_u})\right)$ is minimal (see \sref{Proposition}{prop:minimal_sufficiency}),  \sref{Definition}{def:minimal_sufficiency} states that there exists a bijective function $h:\mathbb{R}\rightarrow \mathbb{R}$ (hence monotonic), such that $\beta^\top \left(F_\phi(X_{t_v}) - F_\phi(X_{t_u})\right) = h(\tau^X_{t_v} - \tau^X_{t_u})$.

\end{proof}

\begin{main_result}
    Using the estimator $\beta^\top F_\phi(X_{t})$, we can identify the true state $\tau^X_{t}$ from $X_{t}$ up to a monotonic transformation $h$.
\label{theorem:identifiability}
\end{main_result}

\begin{proof}
    Let $t_{ref}$ a reference sample time of a new system. Hence, any future sample $X_t$ with $t>t_{ref}$ is sampled from a degraded system. We can assume w.l.o.g. $\tau^X_{t_{ref}}=0$ (since there is no absolute notion of state). From \sref{Lemma}{lemma:minimal_sufficiency} and assuming eq.~\eqref{eq:perfect_g} is achievable (infinite data and $F_\phi$ assumptions), learned parameters $(\beta,\phi)$ are such that $\beta^\top \left(F_\phi(X_{t_{ref}}) - F_\phi(X_{t_v})\right) = h(\tau^X_{t_v})$. Hence, defining the shift constant $C = \beta^\top F_\phi(X_{t_{ref}})$, we get 
    $\beta^\top F_\phi(X_{t_v}) = h(\tau^X_{t_v}) + C$. 
    %

\end{proof}

\section{Related work}
\label{sec:related_work_ctd}

\subsection{Standard trend detection methods}

The trend is any monotone factor underlying temporal data, a ``general direction and tendency" \cite{goldsmith2012monitoring}. It can then be a drift in values, moments, interactions between observed variables, or more generally in the parameters of the generative model from which data have been sampled (i.e., the model of the system $\mathcal{S}$). Different methods may be used depending on the prior information about the trend. We list commonly used trend detection or estimation methods, some recent applications and the relation with our approach.

\paragraph{\textbf{Mann-Kendall test.}} 
It evaluates whether observed values tend to increase or decrease over time using a nonparametric form of monotonic trend regression analysis 
\cite{Mann1945,Kendall1975,Gilbert1987}. 
Mann-Kendall test analyzes the sign of the difference between data samples, with a total of $N(N-1)/2$ possible pairs where $N$ is the number of observations. It accepts missing values, and the data is not required to follow a particular distribution. Hence, the Mann-Kendall statistic for a univariate time series $X$ observed at timesteps $\{t_1, \dots, t_{N}\}$ is:


\begin{equation}
\label{eq:mk_statistics}
    S = \sum_{i=1}^{N-1}\sum_{j=i+1}^N sign (X_{t_i}-X_{t_j})
    \;,
\end{equation}

\noindent where $sign(x)\coloneqq\mathds{1}_{x>0}-\mathds{1}_{x<0}$. 
The hypothesis <<$H_0: \text{no trend}$>> is rejected when $S$ and $2S/(N(N-1)$ are significantly different from $0$. Eq.~\eqref{eq:mk_statistics} is directly related to the CTE problem of eq.~\ref{eq:ranking_problem}. 

\paragraph{\textbf{Finding relation between time and observations.}} It consists in regressing time index $t$ on a response variable $X_t$. For example, $X_t=\beta_0 + \sum_{p=1}^P \beta_p (t)^p + \epsilon$. The null test hypothesis is $H_0:\beta_p=0$ $\forall p$ (i.e., the absence of trend). A more general flexible model use functions of time that can be estimated with smoothing splines, spline regression, or kernel smoother \cite{gray2007comparison}. We note that this method can be used by regressing time not on observation $X$ but on an embedding of $X$, i.e. $F(X_t)=t$. Our CTE implicitly does time regression with an adapted and efficient CL procedure that does not consider absolute time value.

\paragraph{\textbf{Residual of the decomposition of data into stationary components.}} It assumes the existence of stationary generative factors (e.g., cycle or seasonality) and a trend that can be seen as residual. For example, the \textit{empirical mode decomposition} (EMD) \cite{huang1998empirical} is a framework to decompose time series into oscillatory sources $\{c_k\}_{k=1}^K$, whose number of modes is strictly decreasing. By construction, $c_K$ is the trend. CTE method directly filters the information present in $\{c_k\}_{k=1}^{K-1}$ to find $c_K$ (up-to monotonic transform) if $c_K$ is effectively a monotonous factor.

\begin{application}
Climate studies commonly use the methods presented above to extract and analyze trends \cite{AlmendraMartin2022,Huang2022}. Climate data, e.g., hydro-climatic data \cite{Carmona2014} like soil moisture \cite{AlmendraMartin2022} or drought information \cite{Huang2022}, contain significant nonlinear long-term trends. Hence, generic methods like EMD are now commonly used \cite{Huang2022,Wei2018,Carmona2014}. Yet, when data dimensionality or dataset size grows, or when data type is not standard (e.g., large topographic data), EMD finds its limits, despite the recent development of faster EMD methods \cite{zhang2021serial}. The use of NNs in our CTE method, trained with efficient universal procedure, overpasses this problem. CTE can then at least be a comparative method on climate data analysis problems, and at best outperform currently used trend extraction methods.
\qedwhite
\end{application}

\paragraph{\textbf{Monitoring of rolling statistics.}} It consists of computing and monitoring a set of statistics $s^{X_t}$ for all samples $X_t$ at each time $t$. $s^{X_t}$ can be a moment, a covariance (for multivariate data) or autocovariance (for time series data) matrix, a cumulant, parameters of a model \cite{saxena2008turbofan,adamowski2009development}, etc. The sequence $\{s^{X_t}\}_{t \in {1,\dots,N_X}}$ is then monitored to find a trend. Generally, any embedding $F(X_t)$ may be monitored while $F(X_t)$ is informative about $\tau^X$ (e.g., sufficient statistics). The difficulty is to choose the right sample's statistics/embedding where the trend is hidden. Our CTE approach learns the optimal statistics/embedding, then requires no expert definition of the statistics to be monitored. 

\paragraph{\textbf{Blind source separation (BSS).}} It identifies the source factors that generated the observed data, among which the trend if it exists. A common way to do BSS is to apply \textit{independent component analysis} (ICA) \cite{hyvarinen2000independent}. Recent works proposed solutions to the general nonlinear BSS problem using new nonlinear ICA methods based on contrastive learning \cite{hyvarinen2016unsupervised,hyvarinen2019nonlinear}. When structural components are independent, ICA factors include the trend factor (up-to scalar transform, like for CTE). These methods may use NNs as embedding functions and then deal with many data types. Our CTE is a particular case of these nonlinear ICA with a stronger inductive bias towards trend detection. 

\paragraph{\textbf{Slow feature analysis.}} Considering the trend is the slowest non-constant factor underlying data, \textit{slow feature analysis} (SFA) \cite{wiskott2002slow} is a natural solution for trend extraction. It consists of extracting features $F_\phi(X_t)$ from data such that $\|F_\phi(X_t) - F_\phi(X_{t-1}) \|_2$ is the lowest possible, under the constraint that $\{F_\phi(X_{t_i})\}_{i=1}^{N_X}$ components are orthogonal. Recent works showed the relation between SFA and BSS \cite{blaschke2007independent,schuler2019gradient,pineau2020time}. 

\vspace{0.3cm}

Finally, the different standard methods presented above are related to our approach that is more general and universal, and adapts on all situations with all types of data, thanks to the use of NNs and ad-hoc learning procedure. 

\subsection{Relation with survival analysis}
\label{sec:relation_ctd_survival}

A field related to the trend detection problem is the \textit{survival analysis} (SA). The objective of SA is to estimate the lifespan of monitored systems (e.g., the time-to-death of a patient, the time-to-default of a loan, time-to-extinction of species) from data. An illustration of typical survival analysis is given in \autoref{fig:survival}.

\begin{figure}[h]
    \centering
    \includegraphics[width=0.5\linewidth]{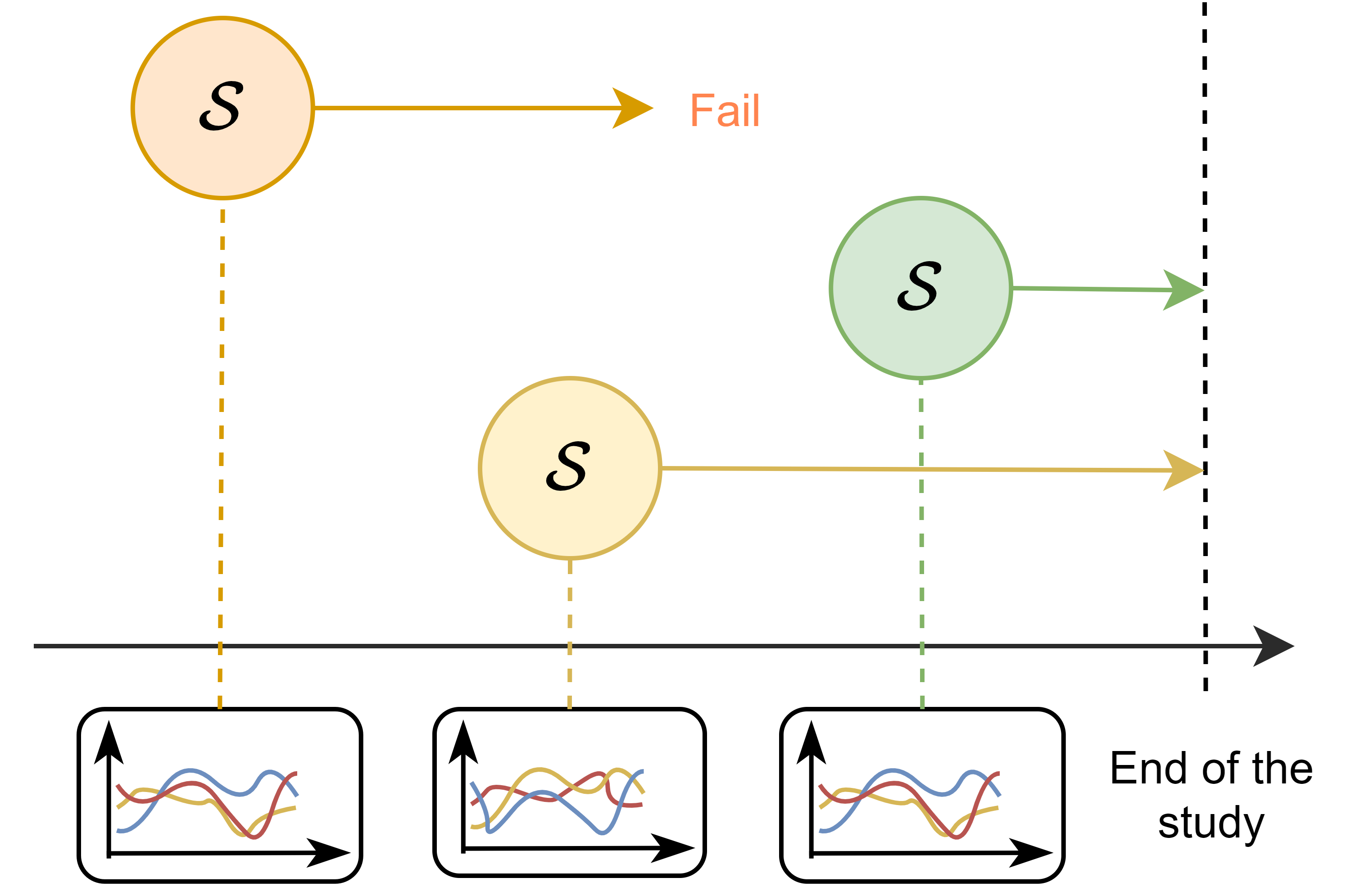}
    \caption{In survival analysis, several systems enter a study by giving a data sample. At the end of the study, we know if a system has failed or has survived. The objective is to predict the lifespan from data using survival information.}
    \label{fig:survival}
\end{figure}

Using previous notations, we have $X_t$ an observed sample at time $t$, $T_X$ the (unknown) lifespan of the system that generated $X$, $T_X - t$ its failure time at time $t$, $\tau_t$ its (unobserved) state at sampling time $t$, $T_X-\tau_t$ its \textit{remaining useful life} (RUL) at time $t$. The objective is to model the conditional \textit{survival function} $S(\tau_t|x)=P(T_X>\tau_t|X_t=x)$, hence the probability that the current state is lower than death-time. A standard way to estimate the survival function $S$ is to use \textit{hazard functions} $h$, defined as $h(\tau_t | x) := - \partial_{\tau_t} \log \left( S(\tau_t|x) \right)$. 

\vspace{0.15cm}

CTE shares several features with survival analysis. First, lifespan is generally assumed to be a monotonic hidden state. Second, the most used performance metric for survival analysis is the \textit{concordance index} (CI) \cite{harrell1996multivariable}, which measures the fraction of pairs of samples that can be correctly ordered in terms of estimated lifespan. It is then directly related to our objective of eq.~\eqref{eq:ranking_problem}. 

\vspace{0.2cm}

\begin{remark}
As illustrated in \autoref{fig:survival}, survival analysis framework generally does not have one system from which a sequence of data is recorded, but several systems from which one data point is recorded (when the systems enter the study). Yet, it is common to assume that all samples have been generated from "equivalent systems", at different moments of its life. Hence, we assume that some of the generative factors, from which we want to disentangle the monotonic factor, represent  individual system's features. 
\end{remark}



The idea of directly training the model to maximize the CI, as in eq.~\eqref{eq:ranking_problem}, exists in survival analysis literature. Authors of \cite{steck2008ranking} describe the ranking problem similarly to eq.~\eqref{eq:ranking_problem}, and relate it to the standard \textit{proportional hazard model} (PHM) $h(\tau | x)=h_0(\tau)\exp(F_\phi(x))$ \cite{cox1972regression}. Nevertheless, it is commonly restricted to linear $F_\phi(x)=\phi^Tx$, for computational reasons and because first order is sufficient in many cases. Authors of \cite{katzman2018deepsurv} use a NN for $F_\phi$ to create a personalized treatment recommendation system, but use a CTE-like function \eqref{eq:cte_model} only for the prediction (i.e., recommendation) part, showing that the subsequent recommendation system, under PHM assumption, is the difference between the embedding of two samples. In \cite{jing2019deep} authors fit a lifespan prediction model learned by pair, which regresses the difference between two samples embedding on the difference between samples target RUL. 

An alternative to PHM is the multi-task logistic regression (MTLR) \cite{yu2011learning}. It consists in building a series of logistic regression models fitted on different time intervals to estimate the probability that the event of interest (e.g., death) happened within each interval. Another alternative to Cox model is the \textit{proportional odds model} (POM) $O(\tau | x)=O_0(\tau)\exp(F_\phi(x))$ \cite{bennett1983analysis}, where $O(\tau|x)$ is the odd of individual surviving beyond time $\tau$. We remind that an odd is the ratio $\frac{S(\tau|x)}{1-S(\tau|x)}$, such that $\exp(F_\phi(x))=\frac{S(\tau|x)}{1-S(\tau|x)}$ for a constant baseline function $O_0(\tau)$. The POM is therefore directly related to the ratio of eq.~\eqref{eq:likelihood_ratio}. 

Yet, we surprisingly found no reference of CTE-like survival analysis with learning procedure on eq.~\eqref{eq:ranking_problem} as described in our paper. 
We therefore apply CTE on standard survival analysis datasets as a practical contribution in \sref{Section}{sec:survival_analysis}. 

\medskip

\section{Experiments}
\label{sec:experiments}

In this section, we propose experiments to illustrate our approach on several datasets with different types of hidden trend. We note $\mathcal{X}^{train}$ and $\mathcal{X}^{test}$ the training and test sets of sequences with hidden trends (validation sets are built from train set). All the results in the tables are computed on the test set. 

We test our CTE method on several datasets, where each observation $X_t$ may be a multivariate time series (Section \ref{sec:balls_springs}) or a sample in survival analysis experiments (Section \ref{sec:survival_analysis}). We finally propose a discussion on the effect of noise on CTE in Section \ref{sec:noisy_cte}. For the two first experiments, we compare with NN-based time series decomposition: NN SFA \cite{schuler2019gradient} (NFSA) and temporal contrastive learning ICA \cite{hyvarinen2016unsupervised} (TCL-ICA). This choice is motivated by the fact that other standard methods (see Section \ref{sec:related_work_ctd}) that are relevant for trend extraction are not adapted to high-dimensional data like images or multivariate time series. 
For the survival analysis, we compare CTE to several models implemented in the PySurvival library \cite{pysurvival_cite}. 

\begin{remarque}
In all experiments, the bold numbers are the best results in term of means; yet taking into account the standard deviations (that can be large), we cannot always claim that CTE is the absolute best model, only that it achieves at least other models' results in several trend factor extraction.
The objective is to illustrate the flexibility of the CTE approach and its universality in trend extraction problems.  
\end{remarque}

\subsection{Ball-springs systems's health monitoring}
\label{sec:balls_springs}
\noindent In this section, we use the same experimental setup than in \cite{pineau2020time}. We will use the model of this paper, called Seq2Graph, as comparison.

\paragraph{\textbf{Dataset.}} We simulate 15000 samples (10000 for train/validation and 5000 for test) from a 10-ball-springs system, consisting of the simultaneous trajectories of 10 identical balls 
in a 2-dimensional space, each ball being randomly connected to some others by springs. 
Each sample $X_{t_i}$, $i \in \llbracket 1, 50 \rrbracket$ is a time series with $10 \times 2$ variables (10 balls in a 2D-space) and 50 time steps. The initialization $X_{t_1}$ of each time series $X_t$ is normally sampled. To simulate the drift, for each sequence $X=\{X_{t_1}, \dots, X_{t_50} \}$, a constant ageing factor $\alpha^X \sim \mathcal{U}([0.9, 1])$ is applied to the system: at each timestep $t$, we randomly choose a spring $(i,j) \in \{1, \dots, 10\}^2$ and multiply its rigidity by $(\alpha^X)^{t}$, i.e. an exponential ageing coefficient with respect to sample time index. Every $50$ samples, we take a new random system, another ageing factor is sampled and we reiterate the simulation of the next sequence of 50 samples. 
The objective is to extract the trend information quantified by $\alpha^X$, without the knowledge that it is hidden in the causal interactions between balls.

\paragraph{\textbf{Results.}} We compare NFSA, TCL-ICA and CTE, plus the model Seq2Graph taken from the source of this experiment \cite{pineau2020unsupervised}. We use a relational NN (RelNN, as in Seq2Graph) and a 3-layer convolutional NN (CNN) as embedding functions $F_\phi$, fitted using Adam optimizer \cite{kingma2014adam}. Results are given in \autoref{tab:springs_results}.

\begin{table}[h]
    \centering
    \begin{tabularx}{0.9\textwidth}{|X|YYYY|}
    \hline
     & Seq2Graph & NFSA & TCL-ICA & CTE  \\
    \hline
    RelNN & $ \bold{0.97} \pm 0.03 $  & $0.93 \pm 0.03 $ & $0.95 \pm 0.02$ & $ \bold{0.97} \pm 0.02 $ \\
    CNN & $ - $  & $0.94 \pm 0.03 $ & $0.60 \pm 0.05$ & $ \bold{0.97} \pm 0.02 $ \\
    \hline
    \end{tabularx}
    \caption{Absolute correlation between estimated and true trend, using two types of embedding neural architecture. Means and standard deviations are computed using 5-fold train/test split, with randomly initialized NNs for each fold.}
    \label{tab:springs_results}
\end{table}

We see that with RelNN architecture, finding the trend hidden in the variable interactions is easy for all methods. Yet, when using a more generic function $F_\phi$, the results of TCL-ICA drops. This sensitivity to embedding function was already unveiled in the experiments of \cite{hyvarinen2016unsupervised}. Naturally, the CTE that is specialized in trend extraction is the more robust and performing approach.

\subsection{Survival analysis experiments}
\label{sec:survival_analysis}

We showed in \sref{Section}{sec:relation_ctd_survival} how the CTE method is related to SA. In this section, we illustrate this relation by applying CTE on survival analysis problems. 

\paragraph{\textbf{Datasets.}} We use four public survival analysis datasets.

\vspace{0.05cm}

Customer churn prediction (Churn) consists in estimating the percentage of customers that stop using the products and services of a company. The survival analysis for customer churn helps companies predicting when a customer is likely to quit considering its personal features. The dataset contains 2000 samples with 12 features, whose $53.4\%$ are right-censored. 

\vspace{0.15cm}

Credit risk (Credit) is the risk carried by a loan company when people borrow money. It corresponds to the likelihood of borrower's credit default with respect to personal features. Survival analysis for credit risk predicts if and when a borrower is likely to fail. The dataset contains 1000 samples with 19 features, whose $30.0\%$ are right-censored. 

\vspace{0.15cm}

Treatment effects on survival time are fundamental for pharmaceutical laboratory. It is possible to do survival analysis of patients. Two public datasets exist. First, German Breast Cancer Study Group (GBCSG2) contain a subset of variables from the German breast cancer study \cite{schumacher1994randomized}. It studies the effects of a treatment with hormones on survival time without cancer recurrence. The dataset contains 686 samples with 8 features, whose $56.4\%$ are right-censored.


\vspace{0.15cm}

Finally, the predictive maintenance of mechanical equipment consists in predicting when an equipment will fail
. We use a public dataset, called Maintenance, whose data is extracted from sensors on manufacturing machines to predict which will fail soon. The dataset contains 1000 samples with 5 features, whose $60.3\%$ are right-censored.

\paragraph{\textbf{Results.}} We compare our CTE survival analysis with five other standard models: the linear and neural Cox proportional hazard models (Cox-PHM) \cite{cox1972regression,katzman2018deepsurv}, the extra-tree and random forest survival analysis (RFS) \cite{ishwaran2008random} and a multi-task logistic regression (MTLR) survival analysis \cite{yu2011learning}. In tree-based survival, an estimation of the \textit{cumulative hazard function} ($\int_t h(t|x)dt$) is done with bags of trees. Other models have been presented in  \sref{Section}{sec:relation_ctd_survival}. We do not provide additional information here. An exhaustive introduction to these models is provided in the website of the PySurvival library \cite{pysurvival_cite} that we used to implement the comparative methods. For these methods, we chose the hyperparameters used in PySurvival tutorials when available. 

\vspace{0.15cm}

For the survival experiments, we used a 10-folds train-test setup. Each dataset is divided into 10 folds used for cross-validation i.e., one fold serves as the testing set while the other ones compose the training set. This 10-folds train-test separation is repeated several times for robustness of the results.  \autoref{tab:survival_results} shows the mean and standard-deviation of the concordance index (CI, see  \sref{Section}{sec:relation_ctd_survival}) computed on the test samples. 

\begin{table}
\centering
\begin{tabularx}{0.99\textwidth}{|X|YYYY|}
    \hline
      & Churn    & Credit & GBCSG2 & Maintenance \\
    \hline
    Linear Cox & $88.1 \pm 1.0$ & $ 77.0 \pm 1.9$ & $66.3 \pm 5.5$ & $96.1 \pm 1.7$ \\
    Neural Cox & $87.5 \pm 2.4$ & $ 75.0 \pm 2.6$ &  $64.4 \pm 4.2$ & $99.3 \pm 1.0$ \\
    Extra Tree & $85.9 \pm 1.2$ & $71.2 \pm 4.3$ & $63.8 \pm 5.4$ & $94.1 \pm 1.5$ \\
    Random Forest & $84.5 \pm 1.5$ & $ 71.5 \pm 3.2$ & $67.6 \pm 4.4$ & $93.1 \pm 2.1$ \\
    Multitask & $89.2 \pm 1.7$ & $ 71.4 \pm 3.7$ & $\textbf{67.9} \pm 7.7$ & $ 93.0 \pm 2.9 $ \\
    CTE (ours) & $\textbf{89.9} \pm 0.9$  & $\textbf{77.2} \pm 1.8$     & $\textbf{67.9} \pm 5.7$ & \textbf{99.6} $\pm$ 0.4 \\
    \hline
\end{tabularx}
\caption{Concordance index for 6 survival analysis models (including CTE) on four datasets. Means and standard deviations are computed using 10-fold train/test split repeated 5 times using randomly initialized NNs (for survival analysis models that use NNs) for each fold.}
\label{tab:survival_results}
\end{table}

\subsection{Noisy trend in CTE}
\label{sec:noisy_cte}

This subsection of the experiments serves as a discussion on the impact of the noise on trend detection, which is standard in real world problems. For example, in climate change trend estimation, the complex and numerous interactions between environment and the variables of interest may perturb the estimator. 

Noisy trend can be characterized by a contamination of the density $p(\tau_{t_u}, \tau_{t_v} | C_{uv})$ with another density $\nu(\tau_{t_u}, \tau_{t_v} | C_{uv};M)$:

\begin{align}
    p^\nu\left(\tau_{t_u}, \tau_{t_v} | C_{uv} \right) := \left(1-\eta\right)p\left(\tau_{t_u}, \tau_{t_v} | C_{uv}\right) + \eta \nu\left(\tau_{t_u}, \tau_{t_v} | C_{uv};M \right)
\label{eq:contamination}
\end{align}

\noindent with $\eta \in [0,1]$ the prevalence of the contamination as named in \cite{fujisawa2008robust}, and $M$ representing the maximum temporal dispersion of the noise (i.e., if $|t_u - t_v|>M$ then $\nu=0$). 


\paragraph{\textbf{Dataset.}} To illustrate the impact of the noise on CTE, we introduce another dataset where trend detection is a useful task: the NASA public Commercial Modular Aero-Propulsion System Simulation dataset (CMAPSS) dataset \cite{saxena2008turbofan}, that consists in a turbine engine degradation and maintenance simulations. We use the dataset FD001 that contains 100 time series of the output of the turbine-engine system, recorded at sea level. Time series are on average 206 time-steps long and have 13 non-constant variables. The engine is operating normally at the start of each time series and develops a fault of unknown initial magnitude in its first moments. We only know that the impact of this fault on the system grows in magnitude until system fails. We extract from these time series sub-trajectories of length 25, with a rolling window with stride 5, to make our dataset.

\begin{figure}[t]
    \centering
    \includegraphics[width=0.24\linewidth]{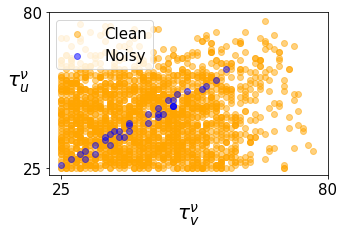}
    \includegraphics[width=0.24\linewidth]{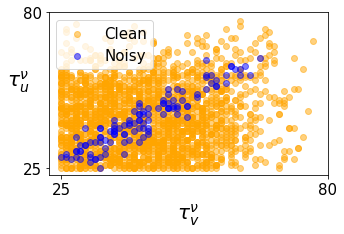}
    \includegraphics[width=0.24\linewidth]{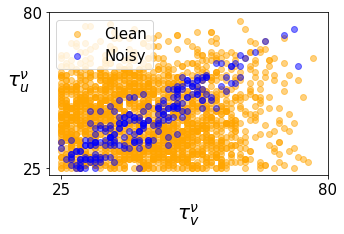}
    \includegraphics[width=0.24\linewidth]{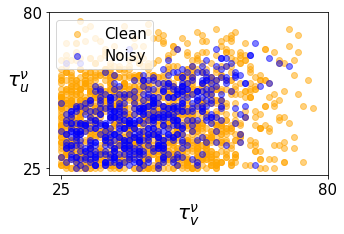}
    \includegraphics[width=0.99\linewidth]{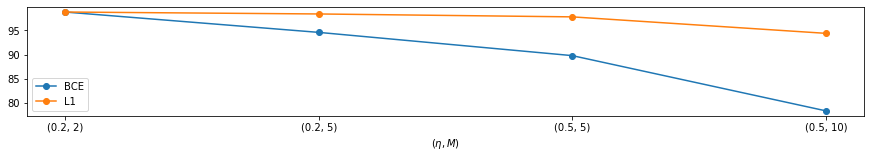}
    \caption{Example of CTE on CMAPSS data with noisy trend. \textbf{Top}: Distribution of couples $(\tau_{t^\nu_u}, \tau_{t^\nu_v})$, for different couples $(\eta, M)$ respectively $(0.2, 2), (0.2, 5), (0.5, 5), (0.5, 10)$ representing the levels of noise. \textbf{Bottom}: mean accuracy of the trend classification for all pairs of samples of the test (unnoisy) set, using loss eq.~\eqref{eq:ranking_problem} versus L1 loss (result from \cite{ghosh2017robust}).}
\label{fig:local_noisy_ctd}
\end{figure}


We contaminate the source with different parameters $(\eta, M)$ in eq. \eqref{eq:contamination}: for each sequence $X\in \mathcal{X}$, we randomly select a proportion $\eta$ of time steps in $\{t_1, \dots, t_{N_X}\}$. For each selected time step $t_i$, we take the $M$ time steps around $t_i$ ($\{t_{i-M/2+1}, \dots, t_{i+M/2}\}$) and randomly shuffle them. We name $t^\nu_t$ this new indexing of the samples, and $\tau^\nu_t:=\tau_{t^\nu}$ the corresponding noisy trend process. We then supervise the CTE with labels $C^\nu_{uv}:=\mathds{1}_{\{t^\nu_u \leq t^\nu_v\}}$, not anymore equal to $C_{uv}$.


\paragraph{\textbf{Results.}} In our CTE framework, noisy trend is equivalent to a \textit{noisy classification} problem, for which a large literature exists \cite{han2020survey}. We use a robust classification losses \cite{ghosh2017robust}, for example a symmetric loss like the L1 loss (see bottom figure in \autoref{fig:local_noisy_ctd}).

\begin{remarque}
In case of \textit{heavy contamination} (e.g., we do not known that the monitored system has been restored), we can use a $\gamma$-crossentropy loss \cite{fujisawa2008robust}. It has been used in \cite{sasaki2020robust} to develop a TCL-ICA robust to source contamination. 
\qedwhite
\end{remarque}
\section{Conclusion}
\label{sec:conclusion}

This paper proposes a new universal trend estimation method using the contrastive learning framework, CTE. Our model is supported by theoretical identifiability results and numerous experimental validations. We showed on several datasets that the strong inductive bias of CTE enables a more robust and accurate trend estimation than two other universal factor extraction methods. Moreover, using the strong relation between trend detection and survival analysis, we applied CTE on survival analysis problems and showed that it outperforms standard models. In further work, we plan to add \textit{interpretability layers} to CTE for industrial reasons, for example, using an attention mechanism in the embedding functions. We also plan to extend our model to treat the trend detection problem when several independent trends are underlying data.




\bibliographystyle{plain}
\bibliography{biblio}

@article{adamowski2009development,
  title={Development of a new method of wavelet aided trend detection and estimation},
  author={Adamowski, Kaz and Prokoph, Andreas and Adamowski, Jan},
  journal={Hydrological Processes: An International Journal},
  volume={23},
  number={18},
  pages={2686--2696},
  year={2009},
  publisher={Wiley Online Library}
}

@inproceedings{banville2019self,
  title={Self-supervised representation learning from electroencephalography signals},
  author={Banville, Hubert and Moffat, Graeme and Albuquerque, Isabela and Engemann, Denis-Alexander and Hyv{\"a}rinen, Aapo and Gramfort, Alexandre},
  booktitle={2019 IEEE 29th International Workshop on Machine Learning for Signal Processing (MLSP)},
  pages={1--6},
  year={2019},
  organization={IEEE}
}

@article{bengio2013representation,
  title={Representation learning: A review and new perspectives},
  author={Bengio, Yoshua and Courville, Aaron and Vincent, Pascal},
  journal={IEEE transactions on pattern analysis and machine intelligence},
  volume={35},
  number={8},
  pages={1798--1828},
  year={2013},
  publisher={IEEE}
}

@article{bennett1983analysis,
  title={Analysis of survival data by the proportional odds model},
  author={Bennett, Steve},
  journal={Statistics in medicine},
  volume={2},
  number={2},
  pages={273--277},
  year={1983},
  publisher={Wiley Online Library}
}

@article{blaschke2007independent,
  title={Independent slow feature analysis and nonlinear blind source separation},
  author={Blaschke, Tobias and Zito, Tiziano and Wiskott, Laurenz},
  journal={Neural computation},
  volume={19},
  number={4},
  pages={994--1021},
  year={2007},
  publisher={MIT Press}
}

@article{chen2017design,
  title={Design of multivariate alarm systems based on online calculation of variational directions},
  author={Chen, Kuang and Wang, Jiandong},
  journal={Chemical Engineering Research and Design},
  volume={122},
  pages={11--21},
  year={2017},
  publisher={Elsevier}
}

@article{choi2005blind,
  title={Blind source separation and independent component analysis: A review},
  author={Choi, Seungjin and Cichocki, Andrzej and Park, Hyung-Min and Lee, Soo-Young},
  journal={Neural Information Processing-Letters and Reviews},
  volume={6},
  number={1},
  pages={1--57},
  year={2005}
}

@article{cox1972regression,
  title={Regression models and life-tables},
  author={Cox, David R},
  journal={Journal of the Royal Statistical Society: Series B (Methodological)},
  volume={34},
  number={2},
  pages={187--202},
  year={1972},
  publisher={Wiley Online Library}
}

@inproceedings{franceschi2019unsupervised,
  title={Unsupervised scalable representation learning for multivariate time series},
  author={Franceschi, Jean-Yves and Dieuleveut, Aymeric and Jaggi, Martin},
  booktitle={Advances in Neural Information Processing Systems},
  pages={4650--4661},
  year={2019}
}

@article{fujisawa2008robust,
  title={Robust parameter estimation with a small bias against heavy contamination},
  author={Fujisawa, Hironori and Eguchi, Shinto},
  journal={Journal of Multivariate Analysis},
  volume={99},
  number={9},
  pages={2053--2081},
  year={2008},
  publisher={Elsevier}
}

@article{goh2001econ,
  title={Econ 2 0A: Sufficiency, Minimal Sufficiency and the Exponential Family of Distributions},
  author={Goh, Chuan},
  year={2001}
}

@article{ghosh2017robust,
  title={Robust loss functions under label noise for deep neural networks},
  author={Ghosh, Aritra and Kumar, Himanshu and Sastry, PS},
  journal={arXiv preprint arXiv:1712.09482},
  year={2017}
}

@book{goldsmith2012monitoring,
  title={Monitoring for conservation and ecology},
  author={Goldsmith, Frank Barrie},
  volume={3},
  year={2012},
  publisher={Springer Science \& Business Media}
}

@article{gray2007comparison,
  title={Comparison of trend detection methods},
  author={Gray, Katharine Lynn},
  year={2007},
  publisher={University of Montana}
}

@article{gutmann2018likelihood,
  title={Likelihood-free inference via classification},
  author={Gutmann, Michael U and Dutta, Ritabrata and Kaski, Samuel and Corander, Jukka},
  journal={Statistics and Computing},
  volume={28},
  number={2},
  pages={411--425},
  year={2018},
  publisher={Springer}
}

@article{han2020survey,
  title={A survey of label-noise representation learning: Past, present and future},
  author={Han, Bo and Yao, Quanming and Liu, Tongliang and Niu, Gang and Tsang, Ivor W and Kwok, James T and Sugiyama, Masashi},
  journal={arXiv preprint arXiv:2011.04406},
  year={2020}
}

@article{harrell1996multivariable,
  title={Multivariable prognostic models: issues in developing models, evaluating assumptions and adequacy, and measuring and reducing errors},
  author={Harrell Jr, Frank E and Lee, Kerry L and Mark, Daniel B},
  journal={Statistics in medicine},
  volume={15},
  number={4},
  pages={361--387},
  year={1996},
  publisher={Wiley Online Library}
}

@article{harvey199310,
  title={10 structural time series models},
  author={Harvey, Andrew C and Shephard, Neil},
  year={1993},
  publisher={Elsevier}
}

@article{huang1998empirical,
  title={The empirical mode decomposition and the Hilbert spectrum for nonlinear and non-stationary time series analysis},
  author={Huang, Norden E and Shen, Zheng and Long, Steven R and Wu, Manli C and Shih, Hsing H and Zheng, Quanan and Yen, Nai-Chyuan and Tung, Chi Chao and Liu, Henry H},
  journal={Proceedings of the Royal Society of London. Series A: mathematical, physical and engineering sciences},
  volume={454},
  number={1971},
  pages={903--995},
  year={1998},
  publisher={The Royal Society}
}

@article{hyvarinen2000independent,
  title={Independent component analysis: algorithms and applications},
  author={Hyv{\"a}rinen, Aapo and Oja, Erkki},
  journal={Neural networks},
  volume={13},
  number={4-5},
  pages={411--430},
  year={2000},
  publisher={Elsevier}
}

@inproceedings{hyvarinen2016unsupervised,
  title={Unsupervised feature extraction by time-contrastive learning and nonlinear ICA},
  author={Hyvarinen, Aapo and Morioka, Hiroshi},
  booktitle={Advances in Neural Information Processing Systems},
  pages={3765--3773},
  year={2016}
}

@inproceedings{hyvarinen2019nonlinear,
  title={Nonlinear ICA using auxiliary variables and generalized contrastive learning},
  author={Hyvarinen, Aapo and Sasaki, Hiroaki and Turner, Richard},
  booktitle={The 22nd International Conference on Artificial Intelligence and Statistics},
  pages={859--868},
  year={2019},
  organization={PMLR}
}

@article{ishwaran2008random,
  title={Random survival forests},
  author={Ishwaran, Hemant and Kogalur, Udaya B and Blackstone, Eugene H and Lauer, Michael S and others},
  journal={The annals of applied statistics},
  volume={2},
  number={3},
  pages={841--860},
  year={2008},
  publisher={Institute of Mathematical Statistics}
}

@article{jing2019deep,
  title={A deep survival analysis method based on ranking},
  author={Jing, Bingzhong and Zhang, Tao and Wang, Zixian and Jin, Ying and Liu, Kuiyuan and Qiu, Wenze and Ke, Liangru and Sun, Ying and He, Caisheng and Hou, Dan and others},
  journal={Artificial intelligence in medicine},
  volume={98},
  pages={1--9},
  year={2019},
  publisher={Elsevier}
}

@article{katzman2018deepsurv,
  title={DeepSurv: personalized treatment recommender system using a Cox proportional hazards deep neural network},
  author={Katzman, Jared L and Shaham, Uri and Cloninger, Alexander and Bates, Jonathan and Jiang, Tingting and Kluger, Yuval},
  journal={BMC medical research methodology},
  volume={18},
  number={1},
  pages={24},
  year={2018},
  publisher={BioMed Central}
}

@article{kingma2014adam,
  title={Adam: A method for stochastic optimization},
  author={Kingma, Diederik P and Ba, Jimmy},
  journal={arXiv preprint arXiv:1412.6980},
  year={2014}
}

@article{le2020contrastive,
  title={Contrastive representation learning: A framework and review},
  author={Le-Khac, Phuc H and Healy, Graham and Smeaton, Alan F},
  journal={IEEE Access},
  year={2020},
  publisher={IEEE}
}

@book{miller2011survival,
  title={Survival analysis},
  author={Miller Jr, Rupert G},
  year={2011},
  publisher={John Wiley \& Sons}
}

@article{niknam2017techniques,
  title={Techniques of trend analysis in degradation-based prognostics},
  author={Niknam, Seyed A and Kobza, John and Hines, J Wesley},
  journal={The International Journal of Advanced Manufacturing Technology},
  volume={88},
  number={9-12},
  pages={2429--2441},
  year={2017},
  publisher={Springer}
}

@article{paszke2017automatic,
  title={Automatic differentiation in pytorch},
  author={Paszke, Adam and Gross, Sam and Chintala, Soumith and Chanan, Gregory and Yang, Edward and DeVito, Zachary and Lin, Zeming and Desmaison, Alban and Antiga, Luca and Lerer, Adam},
  year={2017}
}

@inproceedings{pineau2020time,
  title={Time Series Source Separation with Slow Flows},
  author={Pineau, Edouard and Razakarivony, S{\'e}bastien and Bonald, Thomas},
  booktitle={ICML Workshop on Invertible Neural Networks, Normalizing Flows, and Explicit Likelihood Models},
  year={2020}
}

@inproceedings{pineau2020unsupervised,
  title={Unsupervised ageing detection of mechanical systems on a causality graph},
  author={Pineau, Edouard and Razakarivony, S{\'e}bastien and Bonald, Thomas},
  booktitle={ICMLA},
  year={2020}
}

@Misc{pysurvival_cite,
  author = {Stephane Fotso and others},
  title = {{PySurvival}: Open source package for Survival Analysis modeling},
  year = {2019--},
  url = "https://www.pysurvival.io/"
}

@inproceedings{sasaki2020robust,
  title={Robust contrastive learning and nonlinear ica in the presence of outliers},
  author={Sasaki, Hiroaki and Takenouchi, Takashi and Monti, Ricardo and Hyvarinen, Aapo},
  booktitle={Conference on Uncertainty in Artificial Intelligence},
  pages={659--668},
  year={2020},
  organization={PMLR}
}

@article{saxena2008turbofan,
  title={Turbofan Engine Degradation Simulation Data Set},
  author={Saxena, A and Goebel, K},
  journal={NASA Ames Prognostics Data Repository},
  year={2008},
  publisher={NASA Ames, Moffett Field, CA}
}

@article{schumacher1994randomized,
  title={Randomized 2 x 2 trial evaluating hormonal treatment and the duration of chemotherapy in node-positive breast cancer patients. German Breast Cancer Study Group.},
  author={Schumacher, M and Bastert, G and Bojar, H and Huebner, K and Olschewski, M and Sauerbrei, W and Schmoor, C and Beyerle, C and Neumann, RL and Rauschecker, HF},
  journal={Journal of Clinical Oncology},
  volume={12},
  number={10},
  pages={2086--2093},
  year={1994}
}

@inproceedings{schuler2019gradient,
  title={Gradient-based training of slow feature analysis by differentiable approximate whitening},
  author={Sch{\"u}ler, Merlin and Hlynsson, Hlynur D. and Wiskott, Laurenz},
  booktitle={Asian Conference on Machine Learning},
  pages={316--331},
  year={2019},
  organization={PMLR}
}

@inproceedings{steck2008ranking,
  title={On ranking in survival analysis: Bounds on the concordance index},
  author={Steck, Harald and Krishnapuram, Balaji and Dehing-Oberije, Cary and Lambin, Philippe and Raykar, Vikas C},
  booktitle={Advances in neural information processing systems},
  pages={1209--1216},
  year={2008}
}

@article{thomas2016likelihood,
  title={Likelihood-free inference by ratio estimation},
  author={Thomas, Owen and Dutta, Ritabrata and Corander, Jukka and Kaski, Samuel and Gutmann, Michael U},
  journal={arXiv preprint arXiv:1611.10242},
  year={2016}
}

@inproceedings{wang2015unsupervised,
  title={Unsupervised learning of visual representations using videos},
  author={Wang, Xiaolong and Gupta, Abhinav},
  booktitle={Proceedings of the IEEE International Conference on Computer Vision},
  pages={2794--2802},
  year={2015}
}

@article{wiskott2002slow,
  title={Slow feature analysis: Unsupervised learning of invariances},
  author={Wiskott, Laurenz and Sejnowski, Terrence J},
  journal={Neural computation},
  volume={14},
  number={4},
  pages={715--770},
  year={2002},
  publisher={MIT Press}
}

@inproceedings{yu2011learning,
  title={Learning patient-specific cancer survival distributions as a sequence of dependent regressors},
  author={Yu, Chun-Nam and Greiner, Russell and Lin, Hsiu-Chin and Baracos, Vickie},
  booktitle={Advances in Neural Information Processing Systems},
  pages={1845--1853},
  year={2011}
}

@article{zhang2021serial,
  title={Serial-EMD: Fast empirical mode decomposition method for multi-dimensional signals based on serialization},
  author={Zhang, Jin and Feng, Fan and Marti-Puig, Pere and Caiafa, Cesar F and Sun, Zhe and Duan, Feng and Sol{\'e}-Casals, Jordi},
  journal={Information Sciences},
  volume={581},
  pages={215--232},
  year={2021},
  publisher={Elsevier}
}

@article{Huang2022,
  author = {Huang, Shih-Han and Mahmud, Khalid and Chen, Chia-Jeng},
  title = {Meaningful Trend in Climate Time Series: A Discussion Based On Linear and Smoothing Techniques for Drought Analysis in Taiwan},
  journal = {Atmosphere},
  volume = {13},
  year = {2022},
  number = {3},
  article-number = {444},
  issn = {2073-4433},
  doi = {10.3390/atmos13030444}
}

@article{AlmendraMartin2022,
title = {Analysis of soil moisture trends in Europe using rank-based and empirical decomposition approaches},
journal = {Global and Planetary Change},
volume = {215},
pages = {103868},
year = {2022},
issn = {0921-8181},
doi = {10.1016/j.gloplacha.2022.103868},
author = {Laura, Almendra-Martín and José, Martínez-Fernández and María, Piles and Ángel, González-Zamora and Pilar, Benito-Verdugo and Jaime, Gaona},
}

@article{Carmona2014,
title = {Detection of long-term trends in monthly hydro-climatic series of Colombia through Empirical Mode Decomposition},
journal = {Climatic Change},
volume = {123},
pages = {301–313},
year = {2014},
issn = {1573-1480},
doi = {10.1007/s10584-013-1046-3},
author = {Alejandra M., Carmona and Germán, Poveda},
}

@article{Wei2018,
author = {Wei, Fangli and Wang, Shuai and Fu, Bojie and Pan, Naiqing and Feng, Xiaoming and Zhao, Wenwu and Wang, Cong},
title = {Vegetation dynamic trends and the main drivers detected using the ensemble empirical mode decomposition method in East Africa},
journal = {Land Degradation \& Development},
volume = {29},
number = {8},
pages = {2542-2553},
doi = {10.1002/ldr.3017},
year = {2018}
}

@article{Mann1945,
 ISSN = {00129682, 14680262},
 author = {Henry B. Mann},
 journal = {Econometrica},
 number = {3},
 pages = {245--259},
 publisher = {[Wiley, Econometric Society]},
 title = {Nonparametric Tests Against Trend},
 volume = {13},
 year = {1945}
}

@book{Kendall1975,
 author = {Maurice G. Kendall},
 title = {Rank Correlation Methods.},
 editor = {Charles Griffin},
 adress = {London},
 edition = {4th Edition},
 year = {1975},
}

@book{Gilbert1987,
 author = {Richard O. Gilbert},
 title = {Statistical Methods for Environmental Pollution Monitoring.},
 editor = {John Wiley and Sons},
 adress = {New York},
 year = {1987},
}


\end{document}